\documentclass[nohyperref]{article}

\usepackage{microtype}
\usepackage{graphicx}
\usepackage{subfigure}
\usepackage{booktabs} 

\usepackage{hyperref}
\hypersetup{
    urlcolor=cyan,
    }

\usepackage[accepted]{icml2022}

\usepackage{amsmath}
\usepackage{amssymb}
\usepackage{mathtools}
\usepackage{amsthm}
\usepackage{enumitem}
\usepackage[noorphans,vskip=0ex]{quoting}
\usepackage{xcolor}
\usepackage{xspace}
\usepackage{custom_commands}
\usepackage{bm}
\usepackage{bbm}
\usepackage{etoolbox}
\usepackage{url}

\newcommand{\zerodisplayskips}{%
  \setlength{\abovedisplayskip}{4pt}%
  \setlength{\belowdisplayskip}{4pt}%
  \setlength{\abovedisplayshortskip}{1pt}%
  \setlength{\belowdisplayshortskip}{1pt}}
\appto{\normalsize}{\zerodisplayskips}
\appto{\small}{\zerodisplayskips}
\appto{\footnotesize}{\zerodisplayskips}

\usepackage{titlesec}
\titlespacing\section{-5pt}{0pt plus 2pt minus 2pt}{0pt plus 2pt minus 2pt}
\titlespacing\subsection{-5pt}{0pt plus 2pt minus 2pt}{0pt plus 2pt minus 2pt}
\titlespacing\subsubsection{-5pt}{0pt plus 2pt minus 2pt}{0pt plus 2pt minus 2pt}

\usepackage[capitalize,noabbrev]{cleveref}

\theoremstyle{plain}

\usepackage[textsize=tiny]{todonotes}

\icmltitlerunning{Understanding Gradual Domain Adaptation}

\begin{document}

\twocolumn[
\icmltitle{Understanding Gradual Domain Adaptation:\\Improved Analysis, Optimal Path and Beyond}




\begin{icmlauthorlist}
\icmlauthor{Haoxiang Wang}{uiuc}
\icmlauthor{Bo Li}{uiuc}
\icmlauthor{Han Zhao}{uiuc}
\icmlaffiliation{uiuc}{University of Illinois at Urbana-Champaign, Urbana, IL, USA}

\end{icmlauthorlist}

\icmlcorrespondingauthor{Haoxiang Wang}{hwang264@illinois.edu}

\icmlkeywords{Machine Learning, ICML}

\vskip 0.15in
]



\printAffiliationsAndNotice{}  

\begin{abstract}
The vast majority of existing algorithms for unsupervised domain adaptation (UDA) focus on adapting from a labeled source domain to an unlabeled target domain directly in a one-off way. Gradual domain adaptation (GDA), on the other hand, assumes a path of ($T\mathrm{-}1$) unlabeled intermediate domains bridging the source and target, and aims to provide better generalization in the target domain by leveraging the intermediate ones. Under certain assumptions, \citet{kumar2020understanding} proposed a simple algorithm, \textit{gradual self-training}, along with a generalization bound in the order of $e^{\mathcal O(T)}(\eps_0 \mathrm{+}\mathcal O\bigl(\sqrt {\frac{\log T}{n}}~\bigr) \bigr)$ for the target domain error, where $\eps_0$ is the source domain error and $n$ is the data size of each domain. Due to the exponential factor, this upper bound becomes vacuous when $T$ is only moderately large. In this work, we analyze gradual self-training under more general and relaxed assumptions, and prove a significantly improved generalization bound as $\eps_0\mathrm{+}\cO\bigl(T\Delta \mathrm{+} \frac{T}{\sqrt{n}}\bigr) \mathrm{+} \widetilde{\mathcal O}\bigl(\frac{1}{\sqrt{nT}}\bigr)$, where $\Delta$ is the average distributional distance between consecutive domains. Compared with the existing bound with an \emph{exponential} dependency on $T$ as a \textit{multiplicative} factor, our bound only depends on $T$ \emph{linearly and additively}. Perhaps more interestingly, our result implies the existence of an optimal choice of $T$ that minimizes the generalization error, and it also naturally suggests an optimal way to construct the path of intermediate domains so as to minimize the accumulative path length $T\Delta$ between the source and target. To corroborate the implications of our theory, we examine gradual self-training on multiple semi-synthetic and real datasets, which confirms our findings. We believe our insights provide a path forward toward the design of future GDA algorithms. 
\end{abstract}

\section{Introduction}
It is well known that machine learning models are generally not robust to distribution shifts \citep{sagawa2021extending,koh2021wilds,hendrycks2021natural,gulrajani2021in}. As a result, models trained on one data distribution may have a severe performance drop on test data with a large distribution shift. Unsupervised domain adaptation (UDA) aims to tackle this challenge by adapting models to the test distribution with the help of unlabelled data from the target domain~\citep{ganin2016domain,long2015learning,tzeng2017adversarial,zhao2018adversarial}. 

\begin{figure}[t!]
\begin{center}
\centerline{\includegraphics[width=.8\columnwidth]{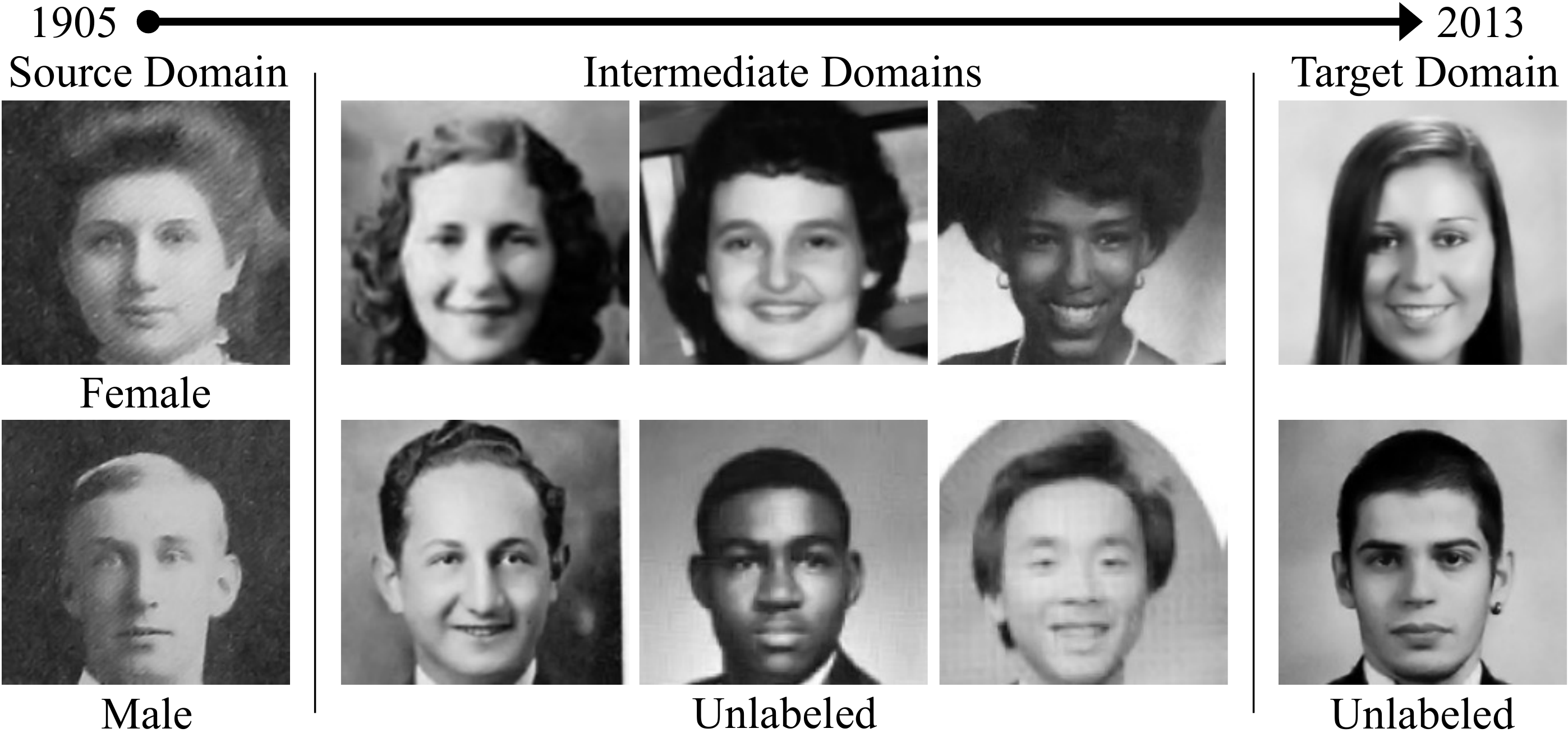}}
\vskip -.1in
\caption{An example of gradual domain adaptation on Portraits \citep{ginosar2015century}, a historical image dataset of US high school yearbook. Given labeled data from a source domain, models are adapted to the target domain, with the help of unlabeled data from intermediate domains gradually shifting from the source to target.}
\vskip -.35in
\label{fig:portraits}
\end{center}
\end{figure}

In general, UDA considers a source domain and a target domain, in which the model is trained and tested, respectively. There is a discrepancy between the two domains, while they share certain similarities. Prior theoretical results show that the generalization error of UDA increases with the larger discrepancy between two domains~\cite{ben-david2010theory,zhao2019domain}. However, empirically, due to the potentially large shift between these two domains, it may be hard to adapt to the target domain in a one-off fashion, and it is observed that existing UDA algorithms do not perform well~\citep{kumar2020understanding,sagawa2021extending,abnar2021gradual} under large shifts. 

Intuitively, one can expect that existing adaptation algorithms perform better under smaller shifts. Thus, one natural idea is to divide a large shift into multiple smaller shifts to mitigate the distribution shift issue. This ``divide-and-conquer'' idea has brought up to a setting known as \textit{gradual domain adaptation} (GDA) \citep{hoffman2014continuous,gadermayr2018gradual,bobu2018adapting,wulfmeier2018incremental,wang2020continuously,kumar2020understanding,abnar2021gradual,chen2021gradual}, where the learner has access to additional unlabeled data from some intermediate domains that \textit{gradually} shift from the source to the target. GDA first fits a model to the source domain, then adapts it to a series of intermediate domains sequentially, and the ultimate goal is to generalize in the target domain. In this way, the large distribution shift is segmented into multiple smaller pieces between neighboring intermediate domains. Notably, in many real-world applications, the data distribution indeed changes gradually over time, distance, or environments \citep{hoffman2014continuous,farshchian2018adversarial,kumar2020understanding,koh2021wilds,malinin2021shifts,zhao2021pyhealth}, in which the unlabelled intermediate domain data are available or can be easily acquired.

\citet{kumar2020understanding} proposes \textit{Gradual Self-Training}, a simple yet effective GDA algorithm, which iteratively applies self-training (ST) on unlabelled data from intermediate domains. Briefly speaking, self-training (i.e., pseudo-labeling) is a popular semi-supervised learning technique that fine-tunes models with self-generated pseudo-labels on unlabelled data \citep{yarowsky1995unsupervised,lee2013pseudo,xie2020self}. The power of gradual self-training is empirically validated on synthetic and real datasets by \citet{kumar2020understanding}, and the authors also provide a theory to demonstrate the improvement of gradual self-training over baselines (i.e., source training only and vanilla self-training) in the presence of large shifts. However, the error bound of gradual self-training provided by \citet{kumar2020understanding} is quite pessimistic and unrealistic in a sense. Given source error $\eps_0$ and $T$ intermediate domains each with $n$ unlabelled data, the bound of \citet{kumar2020understanding} scales as $e^{\mathcal O(T)}\bigl(\eps_0+\cO \bigl(\sqrt {\log (T)/n}\bigr)\bigr)$, which grows exponentially in $T$. This indicates that the more intermediate domains for adaptation, the worse performance that gradual self-training would obtain in the target domain. In contrast, people have empirically observed that a relatively large $T$ is beneficial for gradual domain adaptation \citep{abnar2021gradual,chen2021gradual}. 

Given the sharp gap between existing theory and empirical observations of gradual domain adaptation, we attempt to address the following important and fundamental questions:
\begin{quote}
\itshape
    For gradual domain adaptation, given the source domain and target domain, how does the number of intermediate domains impact the target generalization error? Is there an optimal choice for this number? If yes, then how to construct the optimal path of intermediate domains?
\vspace{-.3em}
\end{quote}
To answer these questions, we focus on gradual self-training \citep{kumar2020understanding}, a representative gradual domain adaptation algorithm, and carry out a novel analysis drastically different from \citet{kumar2020understanding}. Notably, our setting is more general than that of \citet{kumar2020understanding}, in the sense that 1) we have a milder assumption on the distribution shift, 2) we put almost no restriction on the loss function, and 3) our technique applies to all the $p$-Wasserstein distance metrics. As a comparison, existing analysis is restricted to ramp loss\footnote{Ramp loss can be seen as a truncated hinge loss so that it is bounded and more amenable for technical analysis.} and only applies to the $\infty$-Wasserstein metric. At a high level, we first focus on analyzing a pair of consecutive domains, and upper bound the error difference of any classifier over domains bounded by their $p$-Wasserstein distance; then, we telescope this lemma to the entire path over a sequence of domains, and finally obtain an error bound for gradual self-training: $\eps_0\mathrm{+}\cO\bigl(T\Delta \mathrm{+} \frac{T}{\sqrt{n}}\bigr) \mathrm{+} \widetilde{\mathcal O}\bigl(\frac{1}{\sqrt{nT}}\bigr)$, where $\Delta$ is the average $p$-Wasserstein distance between consecutive domains. \looseness=-1

Interestingly, our bound indicates the existence of an optimal choice of $T$ that minimizes the generalization error, which could explain the success of moderately large $T$ used in practice. Notably, the $T\Delta$ in our bound could be interpreted as the length of the path of intermediate domains bridging the source and target, suggesting that one should also consider minimizing the path length $T \Delta$ in practices of gradual domain adaptation. For example, given fixed source and target domains, the path length $T\Delta$ is minimized as the intermediate domains are distributed along some Wasserstein geodesic between the source domain and target domain. We believe these insights on $T$ and $\Delta$ obtained from our error bound are helpful to construct an optimal path of intermediate domains in bridging the source domain and target domain. \looseness=-1

Empirically, we examine gradual self-training on two synthetic datasets (color-shift MNIST \& rotated MNIST) and two real dataset (Portraits \citep{ginosar2015century} \& CoverType \citep{CoverType}). Our experiments validate the insights of our theory: there indeed exists an optimal choice of $T$, and the intermediate domains should be chosen to minimize $T\Delta$. Our empirical observation sheds new light on the importance of constructing the optimal path connecting the source domain and target domain in gradual domain adaptation, and we hope our insight could inspire the design of future gradual domain adaptation algorithms.
\section{Related Work}\label{sec:related-works}

\textbf{Self-Training}~
Self-training, i.e., pseudo-labeling, is a popular semi-supervised learning approach \citep{yarowsky1995unsupervised,lee2013pseudo}, which fine-tunes trained classifiers with pseudo-labels predicted on unlabelled data. There are some common techniques that can enhance self-training, e.g., adding noise to inputs or networks \citep{xie2020self,FixMatch}, progressively assigning pseudo-labels on unlabelled data with high prediction confidence, or applying a strong regularization \citep{arazo2020pseudo,pham2021meta}.

\textbf{Unsupervised Domain Adaptation}~
Unsupervised domain adaptation (UDA) focuses on adapting models trained on labeled data of a source domain to a target domain with unlabelled data. A number of approaches have been proposed for UDA in recent years. Invariant representation learning is one of the most popular approaches \citep{deepCORAL,zhao2019domain}, where adversarial training is usually adopted to learn feature representations invariant between source and target domains \citep{ajakan2014domain,ganin2016domain}. In addition, self-training (i.e., pseudo-labelling) is also adapted for UDA \citep{liang2019distant,liang2020we,zou2018unsupervised,zou2019confidence}. In summary, pseudo-labels of unlabelled target domain data are generated by source-trained classifiers, and they are used to further fine-tune the trained classifiers. Notably, the performance of self-training degrades significantly when there is a large distribution shift between the source and target. \looseness=-1

\textbf{Gradual Domain Adaptation}~
Gradual domain adaptation (GDA) introduces extra intermediate domains to the existing source and target domains of UDA. In general, the intermediate domains shift from the source to the target gradually, such as rotation or time evolution \citep{kumar2020understanding}, and GDA iteratively adapts models from the source to the target along the sequence of intermediate domains. The idea of GDA has been empirically explored in computer vision \citep{gadermayr2018gradual,hoffman2014continuous,wulfmeier2018incremental} with various domain-specific algorithms. Recently, \citet{kumar2020understanding} proposes a general machine learning algorithm for GDA, \textit{gradual self-training}, which outperforms vanilla self-training on several synthetic and real datasets. Moreover, \citet{kumar2020understanding} provides a generalization error bound for gradual self-training, which is the first theoretical guarantee for GDA. In parallel, \citet{wang2020continuously} proposes an adversarial adaptation algorithm for GDA. Later, \citet{abnar2021gradual} proposes a variant of gradual self-training that does not need intermediate domain data, since it could generate pseudo-data for intermediate domains. In GDA, the intermediate domains are given (i.e., ordering domains based on their distance to the source/target domain), \citet{chen2021gradual} removes the requirement of the domain order by proposing a method called Intermediate Domain Labeler (IDOL). \citet{chen2021gradual} shows gradual self-training with IDOL could obtain good performance even without the domain order. Recently, \citet{zhou2022online} proposed an algorithm under the teacher-student paradigm with active query strategy to tackle the gradual adaptation problem. \citet{dong2022algorithms} studied a slightly different setting where labeled data is also available during the intermediate domains. In addition, one may notice that GDA has a setting similar to temporal domain generalization \citep{koh2021wilds,ye2022future}, while the latter has available labels in intermediate domains instead.

\section{Preliminaries}\label{sec:prelim}

\textbf{Notation}~ $\cX,\cY$ denote the input and the output space, and $X,Y$ denote random variables taking values in $\cX,\cY$. In this work, each domain has a data distribution $\mu$ over $\cX \times \cY$, thus it can be written as $\mu = \mu(X,Y)$. When we only consider samples and disregard labels, we use $\mu(X)$ to refer to the sample distribution of $\mu$ over the input space $\cX$.

\subsection{Problem Setup}\label{sec:setup}
\paragraph{Binary Classification}
In this work, we focus on binary classification with labels $\{-1,1\}$. Also, we consider $\mathcal Y$ as a compact space in $\bR$.
\paragraph{Gradually shifting distributions} We have $T\mathrm{+}1$ domains indexed by $\{0,1,...,T\}$, where domain $0$ is the source domain, domain $T$ is the target domain and domain $1,\dots,T\mathrm{-}1$ are the intermediate domains. These domains have distributions over $\cX \times \cY$, denoted as $\mu_{0}, \mu_{1}, \ldots, \mu_{T}$.
\paragraph{Classifier and Loss}
Consider the hypothesis class as $\mathcal H $ and the loss function as $\ell$.
We define the population loss of classifier $h\in \mathcal H$ in domain $t$ as
\begin{align*}
    \eps_{t}(h)\equiv \eps_{\mu_t}(h) \triangleq  \E_{\mu_t}[\ell (h(x),~y)] = \E_{x,y \sim \mu_t}[\ell (h(x),~y)]
\end{align*}
\paragraph{Gradual Domain Adaptation}
In the standard setting of unsupervised domain adaptation (UDA) \citep{zhao2019domain}, a model is trained with $n_0$ labeled data of the source domain and $n$ unlabelled data of the target domain, and it is evaluated by labeled test data of the target domain. In gradual domain adaptation (GDA) \citep{kumar2020understanding}, the model is given additional $T-1$ sequentially indexed intermediate domains, each with $n$ unlabelled data. GDA algorithms usually train the model in the source, then adapt it sequentially over the intermediate domains toward the target. Same as UDA, models trained by GDA are also evaluated by labeled test data of the target domain.

We make a mild assumption on the input data below, which can be easily achieved by data preprocessing. This assumption is common in machine learning theory works \citep{cao2019generalization,fine-grained,rakhlin2014notes}.
\begin{assumption}[Bounded Input Space]\label{assum:input-bound}
Consider the input space $\mathcal X$ is compact and bounded in the $d$-dimensional unit $L_2$ ball, i.e., $\cX \subseteq \{x \in \bR^d: \|x\|_{2}\leq 1\}$. 
\end{assumption}

To quantify distribution shifts between domains, we adopt the well-known Wasserstein distance metric in the Kantorovich formulation \citep{kantorovich1939}, which is widely used in the optimal transport literature \citep{villani2009optimal}.
\begin{definition}[$p$-Wasserstein Distance] Consider two measures $\mu$ and $\nu$ over $\mathbb S \subseteq \bR^d$. For any $p\geq 1$, their $p$-Wasserstein distance is defined as 
\begin{align}
    W_{p}(\mu, \nu):=\left(\inf _{\gamma \in \Gamma(\mu, \nu)} \int_{\mathbb S \times \mathbb S} d(x, y)^{p} \mathrm{~d} \gamma(x, y)\right)^{1 / p}
\end{align}
where $ \Gamma(\mu, \nu)$ is the set of all measures over $\mathbb S \times \mathbb S$ with marginals equal to $\mu$ and $\nu$ respectively.
\end{definition}

In this paper, we consider $p$ as a preset constant satisfying $p\geq 1$. Then, we can use the $p$-Wasserstein metric to measure the distribution shifts between consecutive domains.

\begin{definition}[Distribution Shifts] For $t=1,\dots,T$, denote
\begin{align}
    \Delta_t = W_p(\mu_{t-1}, \mu_{t})
\end{align}
Then, we define the average of distribution shifts between consecutive domains as
\begin{align}
    \Delta = \frac{1}{T} \sum_{t=1}^{T} \Delta_t
\end{align}
\end{definition}

\textbf{Remarks on Wasserstein Metrics}~ The $p$-Wasserstein metric has been widely adopted in many sub-areas of machine learning, such as generative models \citep{WGAN,WAE} and domain adaptation \citep{courty2014domain,courty2016optimal,courty2017joint,redko2019optimal}. Most of these works use $p=1$ or $2$, which is known to be good at quantifying many real-world data distributions \citep{peyre2019computational}. However,  the analysis in \citet{kumar2020understanding} only applies to $p=\infty$, which is uncommon in practice and can lead to a loose upper bound due to the monotonicity property of $W_p$. In fact, for some pairs of data distributions that look close to each other, the $W_\infty$ distance may even become unbounded while $W_1$ and $W_2$ distances are small (e.g., with a few outlier data). 

\subsection{Gradual Self-Training}

The vanilla self-training algorithm (denoted as $\mathrm{ST}$) adapts classifier $h$ with empirical risk minimization (ERM) over pseudo-labels generated on an unlabelled dataset $S$, i.e.,
\begin{align}\label{eq:ST}
    h' = \mathrm{ST}(h,S) = \argmin_{f\in \mathcal H} \sum_{x\in S} \ell(f(x), h(x))
\end{align}
where $h(x)$ represents pseudo-labels provided by the trained classifier $h$, and $h'$ is the new classifier fitted to the pseudo-labels. The technique of hard labelling (i.e., converting $h(x)$ to one-hot labels) is used in some practices of self-training \citep{xie2020self,van2020survey}, which can be viewed as adding a small modification to the loss function $\ell$.

In gradual domain adaptation, we assume a set of $n$ unlabelled data from each intermediate domain and the target domain. Namely, any domain $t\in \{1,\dots, T\}$ has a set of $n$ unlabelled data i.i.d. sampled from $\mu_t$, denoted as $S_t$. For simplicity, we assume the number of labeled data in the source domain is also equal to $n$.

Gradual self-training \citep{kumar2020understanding}, applies self-training to the intermediate domains and the target domain successively, i.e., for $t = 1,\dots, T$,
\begin{align}\label{eq:gradual-ST}
    h_t = \mathrm{ST}(h_{t-1},S_t) = \argmin_{f\in \mathcal H} \sum_{x\in S_t} \ell(f(x), h_{t-1}(x))
\end{align}
where $h_0$ is the model fitted on the source data. $h_T$ is the final trained classifier that is expected to enjoy a low population error in the target domain, i.e., $\eps_{T}$.

\section{Theoretical Analyses}
In this section, we theoretically analyze gradual self-training under assumptions more relaxed than \citet{kumar2020understanding}, and obtain a significantly improved error bound. Our theoretical analysis is roughly split into two steps: (i) we focus on a pair of arbitrary consecutive domains with bounded distributional distance, and upper bound the prediction error difference of any classifier in the two domains by the distributional distance (\cref{lemma:error-diff}); (ii) we view gradual self-training from an online learning perspective, and adopt tools in the online learning literature to analyze the algorithm together with results of step (i), leading to an upper bound (\cref{thm:gen-bound}) of the target generalization error of gradual self-training. Notably, our bound provides several profound insights on the optimal path of intermediate domains used in gradual domain adaptation (GDA), and also sheds light on the design of GDA algorithms. The proofs of all theoretical statements are provided in \cref{supp:proof}.

\subsection{Error Difference over Distribution Shift}\label{sec:error-diff}

Intuitively, gradual domain adaptation (GDA) splits the large distribution shift between the source domain and target domain into smaller shifts that are segmented by intermediate domains. Thus, in the view of reductionism \citep{anderson1972more}, one should understand what happens in a pair of consecutive domains in order to comprehend the entire GDA mechanism.

To start, we adopt three assumptions from the prior work~\citep{kumar2020understanding}\footnote{Assumption \ref{assum:Lipschitz-loss} is not explicitly made by \citet{kumar2020understanding}. Instead, they directly assume the loss function to be ramp loss, which is a more strict assumption than our Assumption \ref{assum:Lipschitz-loss}.}.

\begin{assumption}[$R$-Lipschitz Classifier]\label{assum:Lipschitz-model}
We assume each classifier $h \in \cH$ is $R$-Lipschitz in $\ell_2$ norm, i.e., $\forall x,x' \in \cX$,
\begin{align*}
    |h(x) - h(x')| \leq R \|x-x'\|_2
\end{align*}
\end{assumption}
\begin{assumption}[$\rho$-Lipschitz Loss]\label{assum:Lipschitz-loss}
We assume the loss function $\ell$ is $\rho$-Lipschitz, i.e., $\forall y,y' \in \cY$,
\begin{align}\label{eq:rademacher-assum}
    |\ell(y,\cdot) - \ell(y',\cdot)| \leq \rho \|y-y'\|_2\\ 
    |\ell(\cdot,y) - \ell(\cdot,y')| \leq \rho \|y-y'\|_2
\end{align}
\end{assumption}
\begin{assumption}[Bounded Model Complexity\footnote{This assumption is actually reasonable and not strong. For example, under Assumption \ref{assum:input-bound} and \ref{assum:Lipschitz-model}, linear models directly satisfy \eqref{eq:rademacher-assum}, as proved in \citep{kumar2020understanding,liang2016cs229t}.}]\label{assum:bounded-complexity} 
We assume the Rademachor complexity \citep{bartlett2002rademacher}, $\cR$, of the hypothesis class, $\cH$, is bounded for any distribution $\mu$ considered in this paper. That is, for some constant $B > 0$,
\begin{align*}
\cR_n(\cH; \mu) = \E\left[\sup_{h \in \cH}\frac{1}{n} \sum_{i=1}^n \sigma_i h(x_i) \right]\leq \frac{B}{\sqrt{n}}
\end{align*}
where the expectation is w.r.t. $x_i \sim \mu(X)$ and $\sigma_i \sim \mathrm{Uniform}(\{-1,1\})$ for $i=1,\dots,n$.
\end{assumption}

With these assumptions, we can bound the population error difference of a classifier between a pair of shifted domains in the following proposition. The proof is in \cref{supp:proof:error-diff}.
\begin{lemma}[Error Difference over Shifted Domains]\label{lemma:error-diff}
Consider two arbitrary measures $\mu, \nu$ over $\cX \times \cY$. Then, for arbitrary classifier $h$ and loss function $l$ satisfying Assumption \ref{assum:Lipschitz-model}, \ref{assum:Lipschitz-loss}, the population loss of $h$ on $\mu$ and $\nu$ satisfies
\begin{align}\label{eq:thm:additive-bound:main}
    |\eps_{\mu}(h) - \eps_{\nu}(h)| & \leq \rho \sqrt{R^2 + 1} ~W_p(\mu, \nu)
\end{align}
where $W_p$ is the Wasserstein-$p$ distance metric and $p\geq 1$.
\end{lemma}

Eq. \eqref{eq:gradual-ST} depicts each iteration of gradual self-training with an past classifier $h_t$ and a new one $h_{t+1}$, which are fitted to $S_{t}$ and $S_{t+1}$, respectively. Naturally, one might be curious about how well the performance of $h_{t+1}$ in domain $t\mathrm{+}1$ is compared with $h_{t}$ in domain $t$. We answer this question as follows, with proof in Appendix \ref{supp:proof:algo-stability}.
\begin{proposition}[The stability of the ST algorithm]\label{prop:algorithm-stability}
Consider two arbitrary measures $\mu,\nu$, and denote $S$ as a set of $n$ unlabelled samples i.i.d. drawn from $\mu$. Suppose $h\in \cH$ is a pseudo-labeler that provides pseudo-labels for samples in $S$. Define $\hat h \in \cH$ as an ERM solution fitted to the pseudo-labels,
\vspace{-.8em}
\begin{align}
    \hat h = \argmin_{f \in \mathcal{H}} \sum_{x\in S} \ell(f(x), h(x))
\end{align}
Then, for any $\delta \in (0,1)$, the following bound holds true with probability at least $1-\delta$,
\vspace{-1em}
\begin{align}\label{eq:algo-stability}
    \bigl|\eps_{\mu}(\hat h) \mathrm{-}\eps_{\nu} (h) \bigl| \leq \cO\biggl(W_p(\mu,\nu) \mathrm{+} \frac{\rho B\mathrm{+}\sqrt{\log\frac 1 \delta }}{\sqrt n}~\biggr)
\end{align}
\end{proposition}
\paragraph{Comparison with \citet{kumar2020understanding}}
The setting of \citet{kumar2020understanding} is more restrictive than ours. For example, its analysis is specific to ramp loss \citep{huang2014ramp}, a rarely used loss function for binary classification. \citet{kumar2020understanding} also studies the error difference over consecutive domains, and prove a multiplicative bound (in Theorem 3.2 of \citet{kumar2020understanding}), which can be re-expressed in terms of our notations and assumptions as
\begin{align}\label{eq:kumar-multiplicative}
    \eps_{\mu}(\hat h) \leq \frac{2}{1\mathrm{-}R \Delta_\infty}\eps_{\nu}(h) \mathrm{+} \eps_\mu^* \mathrm{+} \cO\biggl(\frac{\rho B \mathrm{+} \sqrt {\log \frac 1 \delta}}{\sqrt n}\biggr)
\end{align}
where $\eps_\mu^* \triangleq \min_{f\in \cH}\eps_\mu(f)$ is the optimal error of $\cH$ in $\mu$, and $\Delta_\infty\triangleq \max_{y\in\{-1,1\}}(W_\infty(\mu(X|Y=y),\nu(X|Y=y)))$ can be seen as an analog to the $W_p(\mu,\nu)$ in \eqref{eq:algo-stability}. \citet{kumar2020understanding} assumes $1-R\Delta_\infty>0$, thus the error $\eps_\nu(h)$ is increased by the factor $\frac{2}{1-R\Delta_\infty} > 1$ in the above error bound of $\eps_\mu(\hat h)$. This leads to a target domain error bound \textit{exponential} in $T$ (Corollary 3.3. of \citet{kumar2020understanding}) when one applies \eqref{eq:kumar-multiplicative} to the sequence of domains iteratively in gradual self-training (i.e., Eq. \eqref{eq:gradual-ST}). In contrast, our \eqref{eq:algo-stability} indicates $\eps_{\mu}(\hat h ) \leq \eps_{\nu}(h) + \mathrm{other~terms}$, which increases the error $\eps_\nu(h)$ in an \textit{additive} way, leading to a target domain error bound \textit{linear} in $T$.

\paragraph{Remarks on Generality}
\cref{lemma:error-diff} and \cref{prop:algorithm-stability} are not restricted to gradual domain adaptation. Of independent interest, they can be leveraged as useful theoretical tools to handle distribution shifts in other machine learning problems, including unsupervised domain adaptation, transfer learning, OOD robustness, and group fairness.

\subsection{An Online Learning View of GDA}\label{sec:online-view}

One can naively apply \cref{prop:algorithm-stability} to gradual self-training over the sequence of domains (i.e., Eq. \eqref{eq:gradual-ST}) iteratively and obtain an error bound of the target domain as
\begin{align}\label{eq:naive-gen-bound}
    \eps_{T}(h_{T}) \leq \eps_{0}(h_0) + \cO\biggl( T \Delta \mathrm{+} T\frac{\rho B\mathrm{+}\sqrt{\log\frac 1 \delta }}{\sqrt n}~\biggr)
\end{align}
Obviously, the larger $T$, the higher the error bound becomes (this holds even if one assumes $T\Delta\leq \mathrm{constant}$ for fixed source and target domains). However, this contradicts with empirical observations that a moderately large $T$ is optimal \citep{kumar2020understanding,abnar2021gradual,chen2021gradual}.

To resolve this discrepancy, we take an online learning view of gradual domain adaptation, and expect to obtain a more optimistic error bound. Specifically, we consider the domains $t=0,\dots,T$ coming to the model in a sequential way. As the model observes data of domain $t$, it updates itself using ERM over pseudo-labels of the newly observed data, and then it will enter the next iteration of the domain $t+1$. Specifically, the first iteration is the source domain, and the model updates itself using ERM over the labeled data, where the labels can be seen as pseudo-labels generated by a ground-truth labeling function.

To proceed, certain structural assumptions and complexity measures are necessary. For example, VC dimension \citep{vapnik1999nature} and Rademacher complexity \citep{bartlett2002rademacher} are proposed for supervised learning. Similarly, in online learning, Littlestone dimension \citep{littlestone1988learning}, sequential covering number \citep{rakhlin2010online} and sequential Rademacher complexity \citep{rakhlin2010online,rakhlin2015online} are developed as useful complexity measures. To study gradual self-training in an online learning framework, we adopt the framework of \citet{rakhlin2015online}, which views online binary classification as a process in the structure of a \textit{complete binary tree} and defines the \textit{sequential Rademacher complexity} upon that.

\begin{definition}[Complete Binary Trees]\label{def:complete-binary-tree}
We define two complete binary trees $\mathscr X, \mathscr Y$, and the path $\bm \sigma$ in the trees:
\begin{itemize}[leftmargin=0em,align=left,noitemsep,nolistsep]
    \item[] $\mathscr X\triangleq(\mathscr X_0, ..., \mathscr X_{T})$, a sequence of mappings with $\mathscr{X}_t : \{\pm 1\}^{t} \rightarrow \cX$ for $t = 0,...,T$.
    \item[] $\mathscr Y\triangleq(\mathscr Y_0, ..., \mathscr Y_{T})$, a sequence mappings with $\mathscr{Y}_t : \{\pm 1\}^{t} \rightarrow \cY$ for $t = 0,...,T$.
    \item[] $\bm \sigma = (\sigma_0, ..., \sigma_{T}) \in \{\pm 1\}^{t} $, a path in $\mathscr{X}$ or $\mathscr{Y}$.
\end{itemize}

\end{definition}

\begin{definition}[Sequential Rademacher Complexity]\label{def:seq-rademacher}~
Consider $\bm \sigma$ as a sequence of Rademacher random variables and a $t$-dimensional probability vector $\mathbf{q}_t= (q_0,...,q_{t-1})$, then the sequential Rademacher complexity of $\cH$ is
\begin{align*} 
    \mathcal{R}^{\mathrm{seq}}_{t}(\mathcal H) &= \sup_{\bm{\mathscr X,\mathscr Y}} \E_{\bm\sigma}\left[ \sup_{h \in \mathcal H}\sum_{\tau=0}^{t-1} \sigma_\tau q_\tau \ell\bigl( h( { \mathscr X}_\tau (\bm{\sigma})), \mathscr Y_\tau(\bm{\sigma}) \bigr)\right] 
\end{align*}
\end{definition}
To better understand this measure, we present examples of two common model classes, which are provided in \citet{rakhlin2014notes}.

\begin{example}[Linear Models] \label{example:linear-model}
For the linear model class that is $R$-Liphschtiz, i.e., $\cH = \{ x \rightarrow w^\top x: \|w\|_2 \leq R\}$, we have $\mathcal{R}^{\mathrm{seq}}_t(\mathcal H) \leq \frac{R}{\sqrt{t}}$ for $t\in \mathbb{Z}_+$.

\end{example}

\begin{example}[Neural Networks]\label{example:neural-net}
Consider $\mathcal H$ as the hypothesis class of $R$-Lipschitz $L$-layer fully-connected neural nets with $1$-Lipschitz activation function (e.g., ReLU, Sigmoid, TanH). Then, its sequential Rademacher complexity is bounded as $\mathcal{R}^{\mathrm{seq}}_t(\mathcal H) \leq \cO\left(R \sqrt{\frac{\left(\log t \right)^{3(L-1)}}{t}}\right)$ for $t\in \mathbb{Z}_+$.
\end{example}
Besides the model complexity measure, we also adopt a measure of discrepancy among multiple data distributions, which is proposed in works of online learning for time-series data \citep{kuznetsov2014generalization,kuznetsov2015learning,kuznetsov2016time,kuznetsov2017generalization,kuznetsov2020discrepancy}.
\begin{definition}[Discrepancy Measure]\label{def:disc}
For any $t$-dimensional probability vector $\bq_{t} = (q_0,...,q_{t-1})$, the discrepancy measure $\disc(\bq_{t})$ is defined as 
\begin{align}\label{eq:disc-def}
    \disc(\bq_{t}) &= \sup_{h\in \cH} \left( \eps_{t-1}(h) - \sum_{\tau=0}^{t-1} q_\tau \cdot \eps_{\tau}(h) \right) 
\end{align}
\end{definition}

We can further bound this discrepancy in our setting (defined in Sec. \ref{sec:prelim}) as follows. The proof is in \cref{supp:proof:discrepancy-bound}.
\begin{lemma}[Discrepancy Bound]\label{lemma:disc-bound}
With \cref{lemma:error-diff}, the discrepancy measure \eqref{eq:disc-def} can be upper bounded as
\begin{align}\label{eq:desc-upper-bound-1}
    \disc(\bq_{t}) &\leq \rho \sqrt{R^2 + 1} \sum_{\tau=0}^{t-1}q_\tau  (t-\tau - 1)\Delta
\end{align}
With $\bq_{t} = \bq_{t}^* = (\frac{1}{t},..., \frac{1}{t})$, this upper bound can be minimized as
\begin{align}
    \disc(\bq_{t}^*) \leq  \rho \sqrt{R^2 + 1} ~t \Delta / 2= \cO(t\Delta)
\end{align}
\end{lemma}

\subsection{Generalization Bound for Gradual Self-Training}
With our results obtained in \cref{sec:error-diff} and tools introduced in \cref{sec:online-view}, we can prove a generalization bound for gradual self-training within online learning frameworks such as \citet{kuznetsov2016time,kuznetsov2020discrepancy}. However, if we use these frameworks in an off-the-shelf way, the resulting generalization bound will have multiple terms with dependence on $T$ and no dependence on $n$ (the number of samples per domain), since these online learning works do not care about the data size of each domain. This will cause the resulting bound to be loose in terms of $n$. To resolve this, we come up with a novel reductive view of the learning process of gradual self-training, which is more fine-grained than the original view in \citet{kumar2020understanding}. This reductive view enables us to make the generalization bound to depend on $n$ in an intuitive way, which also tightens the final bound. We defer explanations of this view to \cref{supp:proof:gen-bound} along with the proof of \cref{thm:gen-bound}. 

Finally, we prove a generalization bound for gradual self-training that is much tighter than that of \citet{kumar2020understanding}.
\begin{theorem}[Generalization Bound for Gradual Self-Training]\label{thm:gen-bound} 
For any $\delta\in(0,1)$, the population loss of gradually self-trained classifier $h_{T}$ in the target domain is upper bounded with probability at least $1-\delta$ as
\begin{align}
 \eps_{T} (h_{T}) &\leq \sum_{t=0}^T q_t \eps_{t}(h_t) + \|\bq_{\scriptscriptstyle n(T+1)}\|_2 \left(1\mathrm{+}\cO\left(\sqrt{\log   (1/ \delta)}\right)\right)\nonumber\\
 &\qquad \mathrm{+}\disc(\bq_{\scriptscriptstyle T+1})\mathrm{+}\cO\left(\sqrt{\log T } \cR_{n(T+1)}^{\mathrm{seq}}(\ell \circ \cH)\right) \nonumber
 \end{align}
 For the class of neural nets considered in Example \ref{example:neural-net}, 
 \begin{align}
 \eps_{T}(h_T) &\leq  \eps_{0}(h_0) +\cO\biggl(T\Delta \mathrm{+}\frac{T}{\sqrt n} \mathrm{+} T\sqrt{\frac{\log 1 / \delta}{n}}\nonumber\\
 &\mathrm{+}\frac{1}{\sqrt{nT}}  \mathrm{+}\sqrt{\frac{(\log nT)^{3L-2}}{nT}} \mathrm{+} \sqrt{\frac{\log 1/\delta}{nT}}~\biggr)\label{eq:gen-bound}
\end{align}
\end{theorem}
\textbf{Remark}~
The bound in Eq.~\eqref{eq:gen-bound} is rather intuitive: the first term $\eps_{0}(h_0)$ is the source error of the initial classifier, and $T\Delta$ corresponds to the total length of the path of intermediate domains connecting the source domain and the target domain. The asymptotic $\cO(T/\sqrt{n})$ term is due to the accumulated estimation error of the pseudo-labeling algorithm incurred at each step. The $\cO(1/\sqrt{nT})$ term characterizes the overall sample size used by the algorithm along the path, i.e., the algorithm has seen $n$ samples in each domain, and there are $T$ total domains that gradual self-training runs on.

\textbf{Comparison with \citet{kumar2020understanding}}~ Using our notation, the generalization bound of \citet{kumar2020understanding} can be re-expressed as
\begin{align}\label{eq:kumar-gen-bound}
    \eps_{T}(h_T) \leq e^{\cO(T)} \biggl(\eps_{0}(h_0) \mathrm{+} \cO\bigl(\frac{1}{\sqrt n } \mathrm{+} \sqrt{\frac{\log T}{n}}~\bigr)\biggr),
\end{align}
which grows \textit{exponentially} in $T$ as a multiplicative factor. In contrast, our bound \eqref{eq:gen-bound} grows only additively and linearly in $T$, achieving an \textit{exponential improvement} compared with the bound of \citet{kumar2020understanding} shown in \eqref{eq:kumar-gen-bound}.

\begin{figure}[t!]
\begin{center}
\centerline{\includegraphics[width=.9\columnwidth]{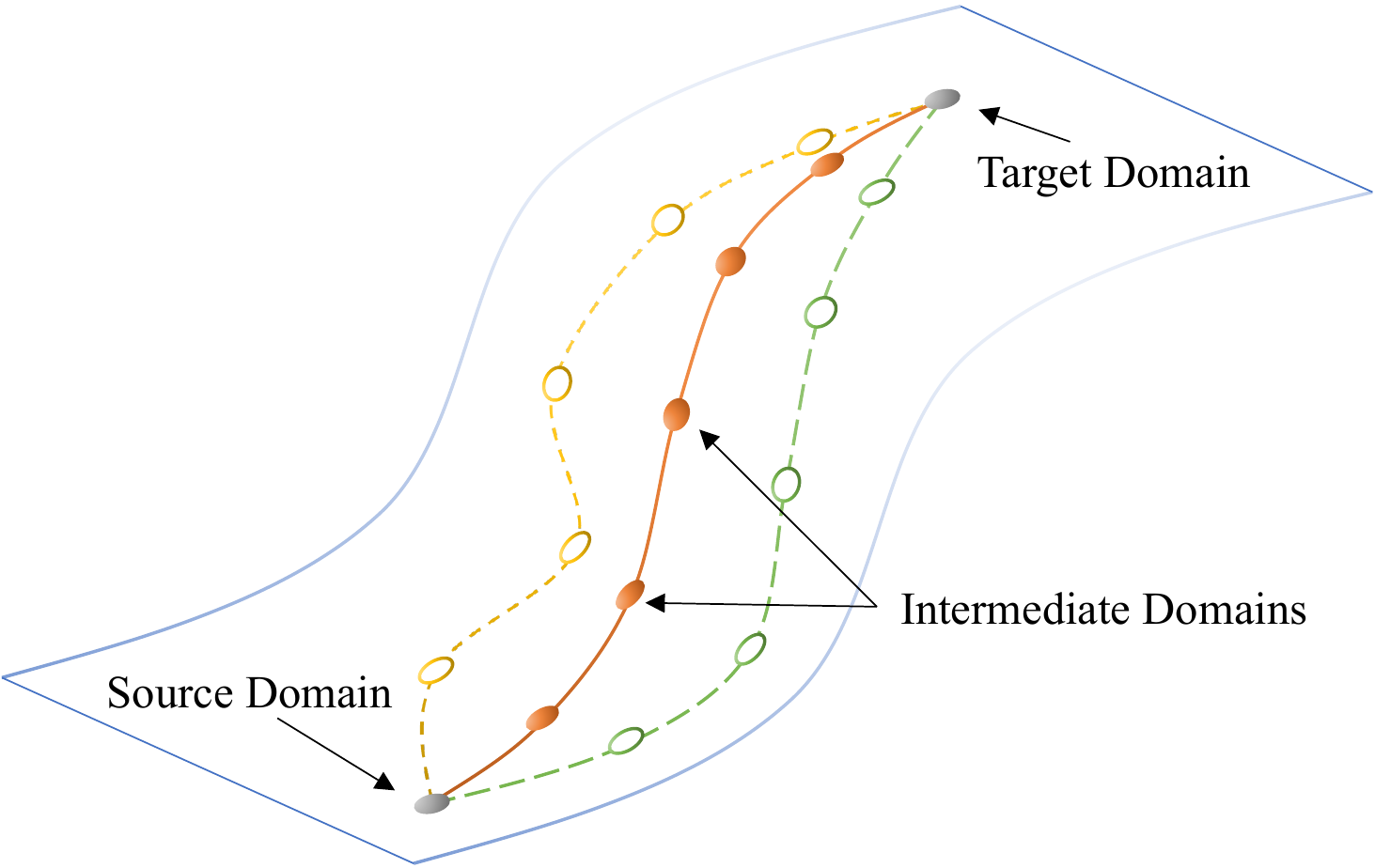}}
\vskip -0.1in
\caption{
An illustration of the optimal path in gradual domain adaptation, with a detailed explanation in Sec.~\ref{sec:optimal-path}. The orange path is the geodesic connecting the source domain and target domain.  
}
\label{fig:manifold-demo}
\end{center}
\vskip -0.4in
\end{figure}

\subsection{Optimal Path of Gradual Self-Training}\label{sec:optimal-path}
It is worth pointing out that our generalization bound in Theorem~\ref{thm:gen-bound} applies to any path connecting the source domain and target domain with $T$ steps, as long as $\mu_0$ is the source domain and $\mu_T$ is the target domain. In particular, if we define $\deltam$ to be an upper bound\footnote{Need to be large enough to ensure that $\mathcal P$ is non-empty.} on the average $W_p$ distance between any pair of consecutive domains along the path, i.e., $\Delta_{\max} \geq \frac{1}{T}\sum_{t=1}^T W_p(\mu_{t-1}, \mu_t)$, and let $\mathcal{P}$ to be the collection of paths with $T$ steps connecting $\mu_0$ and $\mu_T$:
\begin{equation*}
\mathcal{P}\defeq \{(\mu_{t})_{t=0}^T\mid \frac 1 T \sum_{t=1}^T W_p(\mu_{t-1}, \mu_t)\leq \deltam\}~,
\end{equation*}
then we can extend the generalization bound in Theorem~\ref{thm:gen-bound}:
\begin{equation}
\eps_{T}(h_T) \leq  \eps_{0}(h_0) \mathrm{+} \inf_{\mathcal{P}} \widetilde{\cO}\biggl(T\deltam \mathrm{+}\frac{T}{\sqrt n} \mathrm{+} \sqrt{\frac{1}{nT}}~\biggr)
\label{eq:simplified}
\end{equation}
Minimizing the RHS of the above upper bound w.r.t.\ $T$ (the proof is provided in \cref{supp:proof:optimal-T}), we obtain the optimal choice of $T$ on the order of
\begin{align}\label{eq:optimal-T-expression}
    \widetilde{\cO}\left(\left(\frac{1}{1 + \deltam \sqrt{n}}\right)^{2/3}\right).
\end{align}
However, the above asymptotic optimal length may not be achievable, since we need to ensure that $T\deltam$ is at least the length of the geodesic connecting the source domain and target domain. To this end, define $L$ to be the $W_p$ distance between the source domain and target domain, we thus have the optimal choice $T^*$ as
\begin{equation}
    T^* = \max\left\{\frac{L}{\deltam}, \widetilde{\cO}\left(\left(\frac{1}{1 + \deltam \sqrt{n}}\right)^{2/3}\right)\right\}.
\label{equ:optt}
\end{equation}
Intuitively, the inverse scaling of $T^*$ and $\deltam$ suggests that, if the average distance between consecutive domains is large, it is better to take fewer intermediate domains. 

\textbf{Illustration of the Optimal Path}~
To further illustrate the notion of the optimal path connecting the source domain and target domain implied by our theory, we provide an example in Fig.~\ref{fig:manifold-demo}. Consider the metric space induced by $W_p$ over all the joint distributions with finite $p$-th moment, where both the source and target could be understood as two distinct points. In this case, there are infinitely many paths of step size $T$ connecting the source and target, such that the average pairwise distance is bounded by $\deltam$. Hence, one insight we can draw from Eq.~\eqref{eq:simplified} is that: if the learner could construct the intermediate domains, then it is better to choose the path that is as close to the geodesic, i.e., the shortest path between the source and target (under $W_p$), as possible. This key observation opens a broad avenue forward toward algorithmic designs of gradual domain adaptation to \textit{construct} intermediate domains for better generalization performance in the target domain. However, the design and discussion of algorithms along this direction are beyond the scope of this paper, and we leave it to future works.

\section{Experiments}
To empirically validate our theoretical findings, we examine gradual self-training on two synthetic and two real datasets. 

\textbf{Implementation}~
We adopt the official code of gradual self-training by \citet{kumar2020understanding} that is implemented in Keras \citep{keras} and TensorFlow \citep{tensorflow}. The Color-Shift MNIST and CoverType datasets are not covered by the official code of \citet{kumar2020understanding}, thus we include them by ourselves. We adopt the data splitting of CoverType from \citet{kumar2020understanding}, and our data splitting of Rotated MNIST and Portraits are mostly the same as \citet{kumar2020understanding}.
Following \citet{kumar2020understanding}, we use the architecture of 3-layer ReLU MLP with BatchNorm \citep{batchnorm} and Dropout(0.5) \citep{dropout}, and apply it to all datasets. We use the cross-entropy loss and the Adam optimizer \citep{adam} following the practices of \citet{kumar2020understanding}.

\begin{figure*}[t!]
\begin{center}
\includegraphics[width=.92\columnwidth]{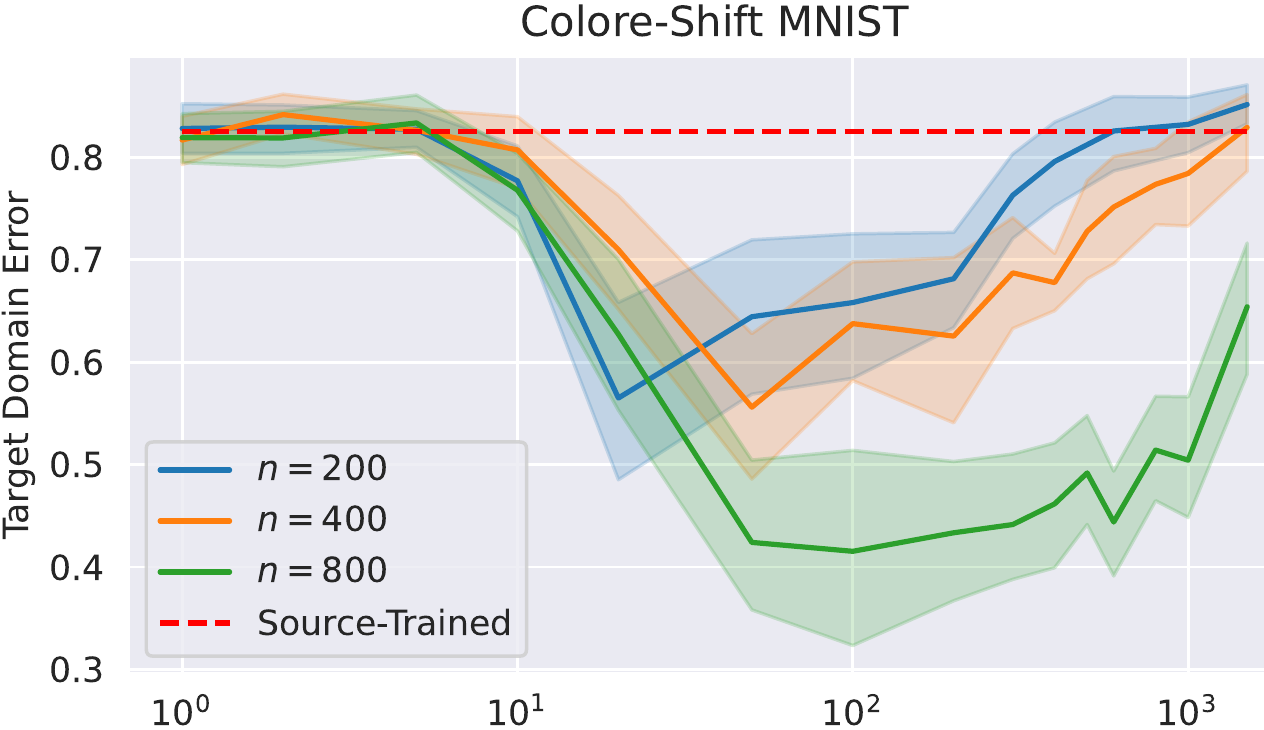}
\includegraphics[width=.95\columnwidth]{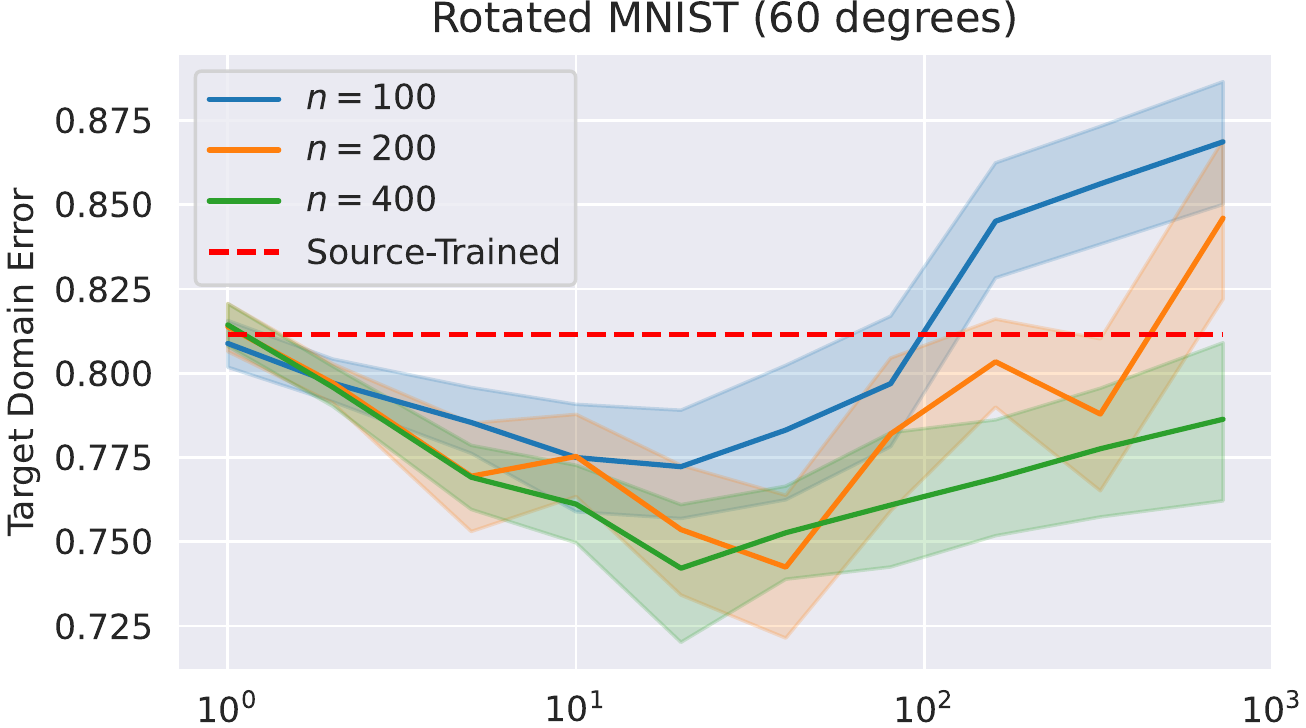}
\includegraphics[width=.95\columnwidth]{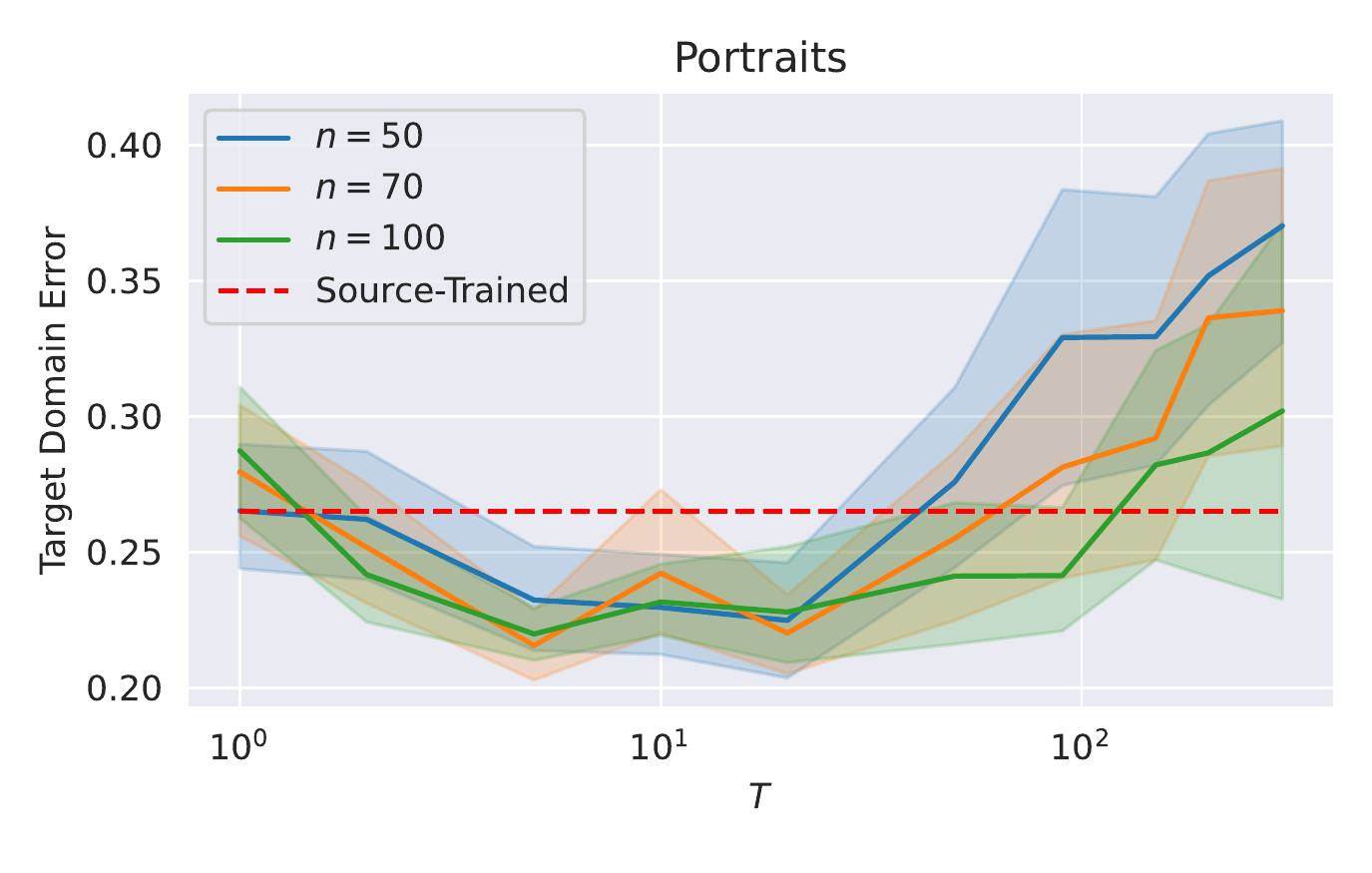}
\includegraphics[width=.95\columnwidth]{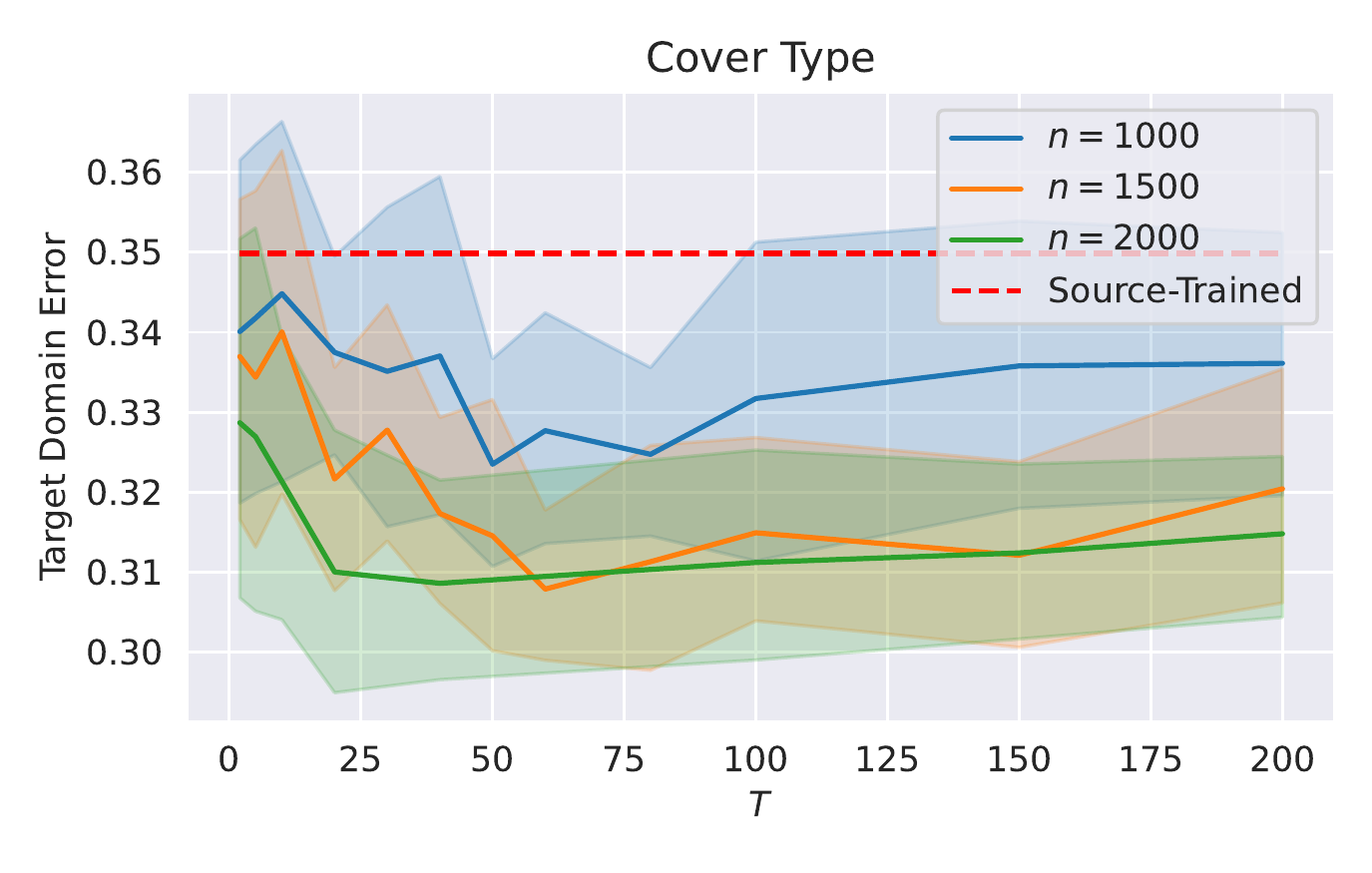}
\vspace{-1em}
\caption{
Empirical results of gradual self-training with increasing $T$ on four datasets. The red dashed line is for models that are purely trained on the source domain and evaluated in the target domain. In each plot, results of three choices of $n$ (number of samples per intermediate domain) are provided. The shading of each curve indicates the standard deviation of measured error over 20 runs.
}
\label{fig:main-exp}
\end{center}
\vskip -.2in
\end{figure*}

\textbf{Datasets}~
We use two semi-synthetic datasets built upon MNIST \citep{mnist}: Color-Shift MNIST and Rotated MNIST. Color-Shift MNIST is our custom dataset that shifts the pixel color values of grayscale MNIST images, and Rotated MNIST is a common dataset for gradual domain adaptation studies \citep{kumar2020understanding,abnar2021gradual,chen2021gradual}. We also consider two real datasets, CoverType and Portraits, which are used by prior GDA works \citep{kumar2020understanding,chen2021gradual}.

\textit{Color-Shift MNIST}: We normalize the pixel value of each MNIST image from [0,255] to [0,1]. The source domain is of the original MNIST data distribution, and the target domain contains images with pixels shifted by +1, i.e., the pixel value range is shifted from [0,1] to [1,2]. The 50K training set images of MNIST are split into a source domain of 5K images (no shift) and intermediate domains of the rest data (shifted by a value between 0 and +1 uniformly). 10K MNIST test images are all shifted by +1 to the target.

\textit{Rotated MNIST}: A semi-synthetic dataset rotating MNIST images by an angle between 0 and 60 degrees. The 50K training set images of MNIST are divided into a source domain of 5K images (no rotation), intermediate domains of 42K images (0-60 degrees), and a set of validation data of the rest images. The 10K MNIST test set images are all rotated by 60 degrees to become data of the target domain.

\textit{CoverType} \citep{CoverType}:
A tabular dataset hosted by the UCI repository \citep{UCI_repository} that aims to predict the forest cover type given 54 features. Following \citet{kumar2020understanding}, we sort examples by their distances to the water body in ascending order, and split the data into a source domain (first 50K examples), intermediate domains (following 400K examples), and a target domain (the last 50K examples).

\textit{Portraits} \citep{ginosar2015century}:
An image dataset of grayscale photos of high school seniors from 1905 to 2013 (Fig. \ref{fig:portraits}). We split the dataset into a source domain of 1905-1930 (1K images), intermediate domains of 1931-2009 (34K images), and a target domain of 2010-2013 (1K examples).

\textbf{Empirical Results}~
We run gradual self-training with various choices of $n$ and $T$ over these four datasets (i.e., there are $T-1$ intermediate domains and $1$ target domain, each with $n$ unlabelled data for adaptation). Each curve on the figure is the average accuracy measured over 20 repeated experiments. Empirical results shown in Fig. \ref{fig:main-exp} are consistent with the theoretical prediction of our generalization bound that we discuss in Sec. \ref{sec:optimal-path}: as the source and target are fixed, along a chosen path of intermediate domains (e.g., counter-clockwise rotation in Rotated MNIST from 0 to 60 degrees), the target domain test error decreases then increases, indicating the existence of an optimal choice of $T$ for each $n$ of consideration.

\textbf{Remarks}~
Our empirical results suggest that in practices of gradual domain adaptation, the hyper-parameters $T$ and $n$ are crucial and should be carefully treated. Also, if one can collect or generate intermediate domain data \citep{abnar2021gradual}, the choices of intermediate domains should be examined in advance, and our theoretical findings in this paper could serve as a guide.

\section{Conclusion and Discussion}
For unsupervised gradual domain adaptation, we provide a significantly improved analysis for the generalization error of the gradual self-training algorithm, under a more general setting with relaxed assumptions. In particular, compared with existing results, our bound provides an \emph{exponential} improvement on the dependency of the step size $T$, as well as a better sample complexity of $O(1/\sqrt{nT})$, as opposed to $O(1/\sqrt{n})$ as in the existing work. Perhaps more interestingly, our generalization bound contains one term that admits a natural and intuitive interpretation: the length of the path produced by the intermediate domains that connect the source domain and target domain. Hence, our theory indicates that when constructing intermediate domains, an algorithm should aim to find those on the geodesic connecting the source domain and target domain. We believe this insight can open a broad avenue toward future algorithm designs for gradual domain adaptation when no intermediate domains are available. It also remains an open question on how to efficiently construct the optimal intermediate domains when only unlabeled data are available from the target domain.

\section*{Acknowledgements}
This work is partially supported by NSF grant No.1910100, NSF CNS No.2046726, C3 AI, and the Alfred P. Sloan Foundation. BL and HZ would like to thank the support from a Facebook research award.

\bibliography{reference.bib}
\bibliographystyle{icml2022}

\newpage
\appendix
\onecolumn

\section{Proof}\label{supp:proof}

\subsection{Proof of \cref{lemma:error-diff} (Error Difference over Shifted Domains)}\label{supp:proof:error-diff}
Restatement of \cref{lemma:error-diff}.
\textit{Consider two arbitrary measures $\mu, \nu$ over $\cX \times \cY$. Then, for arbitrary classifier $h$ and loss function $l$ satisfying Assumption \ref{assum:Lipschitz-model}, \ref{assum:Lipschitz-loss}, the population loss of $h$ on $\mu$ and $\nu$ satisfies
\begin{align}\label{eq:supp:thm:additive-bound:main}
    |\eps_{\mu}(h) - \eps_{\nu}(h)| & \leq \rho \sqrt{R^2 + 1} ~W_p(\mu, \nu)
\end{align}
where $W_p$ is the Wasserstein-$p$ distance metric and $p\geq 1$.}

\begin{proof}
The population error difference of $h$ over the two domains (i.e., $\mu$ and $\nu$ is
\begin{align}
    |\eps_{\mu}(h) - \eps_{\nu}(h)| &= \left|\E_{x,y\sim \mu}[\ell(h(x), y)] - \E_{x',y'\sim\nu}[\ell(h(x'),y')]\right| \nonumber \\
    &= \left| \int \ell(h(x),y) d \mu - \int \ell(h(x'),y') d \nu  \right|\label{eq:err-diff}
\end{align}

Let $\gamma $ be an arbitrary coupling of $\mu$ and $\nu$, i.e., it is a joint distribution with marginals as $\mu$ and $\nu$. Then, \eqref{eq:err-diff} can be re-written and bounded as
\begin{align}
    |\eps_{\mu}(h) - \eps_{\nu}(h)| &= \left | \int \ell(h(x),y) - \int \ell(h(x'),y') d \gamma  \right|\\
    \text{(triangle inequality)}&\leq \int  \left |\ell(h(x),y) - \int \ell(h(x'),y') \right|d \gamma \\
    \text{($\ell$ is $\rho$-Lipschitz)} &\leq \int  \rho \left( \|h(x) - h(x')\| + \|y-y'\| \right) d\gamma \\
    \text{($h$ is $R$-Lipschitz)} &\leq \int  \rho  R\|x - x'\| +\rho \|y-y'\|  d\gamma \\
    \text{($R > 0$)} &\leq \int  \rho \sqrt{R^2 + 1} \left(\|x -x'\| + \|y-y'\|\right)  d\gamma 
\end{align}
Since $\gamma$ is an arbitrary coupling, we know that 
\begin{align}
    |\eps_{\mu}(h) - \eps_{\nu}(h)| &\leq \inf_{\gamma}  \int  \rho \sqrt{R^2 + 1} \left(\|x -x'\| + \|y-y'\|\right)  d\gamma \\
    &= \rho \sqrt{R^2 + 1} W_1(\mu, \nu)
\end{align}

Since the Wasserstein distance $W_p$ is monotonically increasing for $p \geq 1$, we have the following bound, 
\begin{align}\label{eq:error-diff-neighbour}
    |\eps_{\mu}(h) - \eps_{\nu}(h)|\leq \rho \sqrt{R^2 + 1} W_1(\mu, \nu) \leq \rho \sqrt{R^2 + 1} W_p(\mu, \nu) \leq \rho \sqrt{R^2 + 1} 
\end{align}
\end{proof}

\subsection{Proof of \cref{prop:algorithm-stability} (Algorithm Stability)}\label{supp:proof:algo-stability}

Restatement of \cref{prop:algorithm-stability}.
\textit{Consider two arbitrary measures $\mu,\nu$, and denote $S$ as a set of $n$ unlabelled samples i.i.d. drawn from $\mu$. Suppose $h\in \cH$ is a pseudo-labeler that provides pseudo-labels for samples in $S$. Define $\hat h \in \cH$ as an ERM solution fitted to the pseudo-labels,
\begin{align}
    \hat h = \argmin_{f \in \mathcal{H}} \sum_{x\in S} \ell(f(x), h(x))
\end{align}
Then, for any $\delta \in (0,1)$, the following bound holds true with probability at least $1-\delta$,
\begin{align}\label{eq:supp:algo-stability}
    \bigl|\eps_{\mu}(\hat h) \mathrm{-}\eps_{\nu} (h) \bigl| \leq \cO\biggl(W_p(\mu,\nu) \mathrm{+} \frac{\rho B\mathrm{+}\sqrt{\log\frac 1 \delta }}{\sqrt n}~\biggr)
\end{align}
}

\begin{proof}
Define $\wh \eps_{\mu}(h) \coloneqq \frac{1}{|S|}\sum_{x\in S} \ell(h(x), y)$ as the empirical loss over the dataset $S$, where $S$ consists of samples i.i.d. drawn from $\mu(X)$ and $y$ is the ground truth label of $x$.

Then, we have the following sequence of inequalities:
\begin{align*}
    (\text{Use Lemma A.1 of \citet{kumar2020understanding}})\quad \eps_{\mu}(h) &\leq \wh \eps_{\mu}( \hat h)+ \cO\left(R_n(\ell \circ \mathcal{H})+\sqrt{\frac{\log(1/\delta)}{n}}\right)\\
    \left(\text{since } h(x)=\hat h(x) ~\forall x \in S\right)~&= \widehat \eps_{\mu}(h)+ \cO\left(R_n(\ell \circ \mathcal{H})+\sqrt{\frac{\log(1/\delta)}{n}}\right)\\
    (\text{Use Lemma A.1 of \citet{kumar2020understanding} again})~&\leq \eps_{\mu} (h) + \cO\left(2 R_n(\ell \circ \mathcal{H})+2 \sqrt{\frac{\log(1/\delta)}{n}}\right)\\
    (\text{By \cref{lemma:error-diff}})&\leq \eps_{\nu} (h) + \rho \sqrt{R^2+1} W_p(\mu,\nu) \\
    &\quad + \cO\left( R_n(\ell \circ \mathcal{H})+\sqrt{\frac{\log(1/\delta)}{n}}\right)\\
    (\text{By Talagrand's lemma with Assumption \ref{assum:Lipschitz-loss},\ref{assum:bounded-complexity}})&\leq \eps_{\nu} (h) + \rho \sqrt{R^2+1} W_p(\mu,\nu) \\
    &\quad + \cO\left( \frac{\rho B}{\sqrt n} +\sqrt{\frac{\log(1/\delta)}{n}}\right)\\
    &\leq \eps_{\nu} (h)  + \cO\left( W_p(\mu,\nu)+ \frac{\rho B}{\sqrt n} +\sqrt{\frac{\log(1/\delta)}{n}}\right)
\end{align*}

For the step using Talagrand's lemma \citep{talagrand1995concentration}, the proof of Lemma A.1 of \citet{kumar2020understanding} also involves an identical step, thus we do not replicate the specific details here. 
\end{proof}

\subsection{Proof of \cref{lemma:disc-bound} (Discrepancy Bound)}\label{supp:proof:discrepancy-bound}
Restatement of \cref{lemma:disc-bound}.
\textit{With \cref{lemma:error-diff}, the discrepancy measure \eqref{eq:disc-def} can be upper bounded as
\begin{align}\label{eq:supp:desc-upper-bound-1}
    \disc(\bq_{t}) &\leq \rho \sqrt{R^2 + 1} \sum_{\tau=0}^{t-1}q_\tau\cdot (t-\tau-1)\Delta
\end{align}
Choosing $\bq_{t} = \bq_{t}^* = (\frac{1}{t},..., \frac{1}{t})$, this upper bound can be minimized as
\begin{align}
    \disc(\bq_{t}^*) \leq  \rho \sqrt{R^2 + 1} ~t \Delta / 2= \cO(t\Delta)
\end{align}
}

\begin{proof}
Within our setup of gradual self-training,
\begin{align*}
    \disc(\bq_{t}) &= \sup_{h\in \cH} \left( \eps_{t-1}(h) - \sum_{\tau=0}^{t-1} q_\tau \cdot \eps_{\tau}(h) \right) \\
    &= \sup_{h\in \cH} \left(  \sum_{\tau=0}^{t-1} q_\tau \left(\eps_{t-1}(h) -  \eps_{\tau}(h) \right)\right)\\
    &\leq \sup_{h\in \cH} \left(  \sum_{\tau=0}^{t-1} q_\tau |\eps_{t-1}(h) -  \eps_{\tau}(h) |\right)\\
    (\text{By \cref{lemma:error-diff}}) &\leq \rho \sqrt{R^2 + 1} \sum_{\tau=0}^{t-1}q_\tau\cdot (t-\tau-1)\Delta
\end{align*}
With $\bq_{t} = \bq_{t}^* = (\frac{1}{t},..., \frac{1}{t})$, this bound becomes
\begin{align*}
    \disc(\bq_t^*) \leq \rho \sqrt{R^2 + 1} \sum_{\tau=0}^{t-1}q_\tau\cdot (t-\tau-1)\Delta = \rho \sqrt{R^2 + 1}~ \frac{t}{2} \Delta = \cO(t\Delta)
\end{align*}
and it is trivial to show that this upper bound is smaller than any other $\bq_t$ with $\bq_t \neq \bq_t^*$.
\end{proof}

\subsection{Proof of \cref{thm:gen-bound} (Generalization Bound for Gradual Self-Training))}\label{supp:proof:gen-bound}

\begin{figure}[t!]
\begin{center}
\centerline{\includegraphics[width=.9\textwidth]{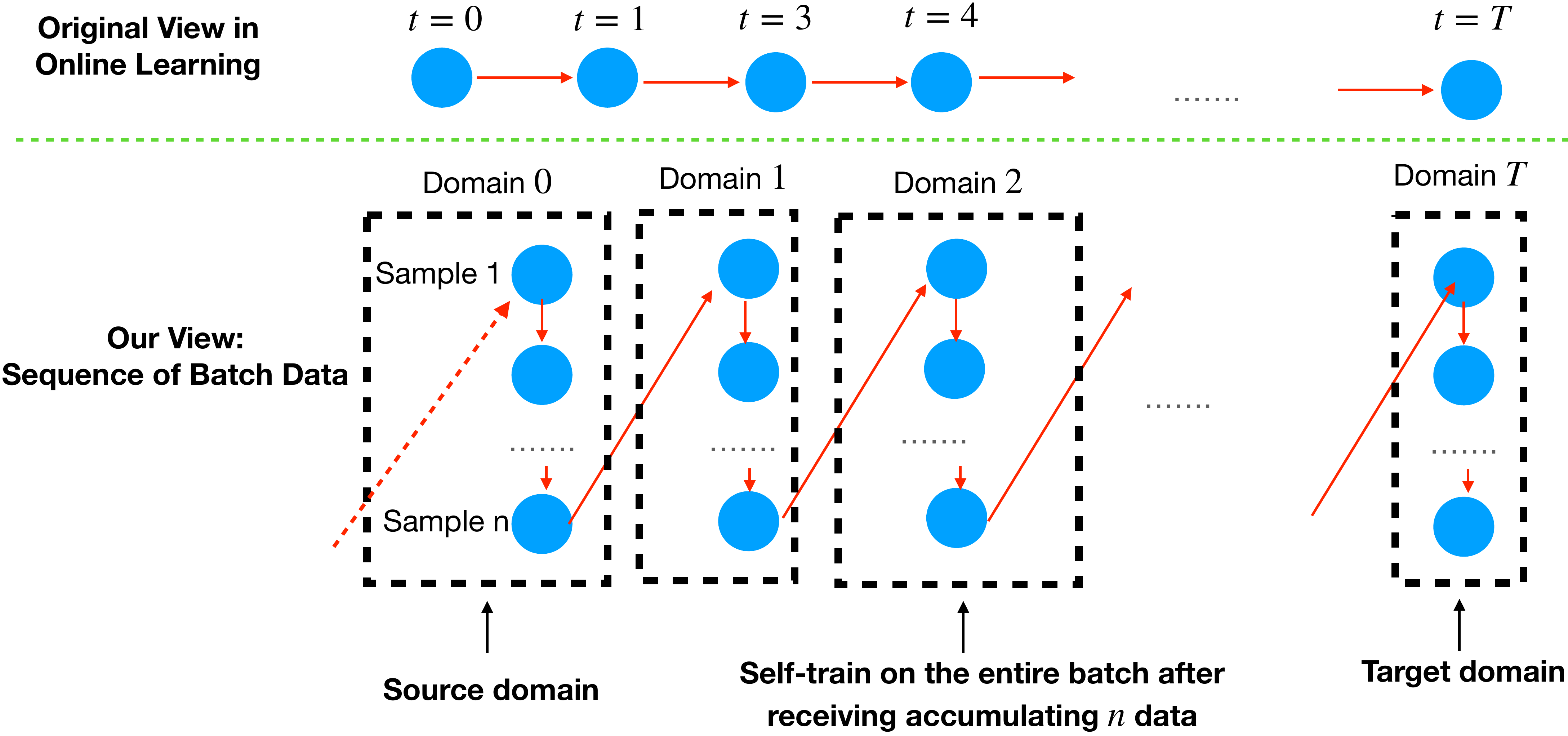}}
\caption{Our reductive view of gradual self-training that is helpful to \cref{thm:gen-bound}.}
\label{fig:our-view}
\end{center}
\vskip -.4in
\end{figure}

Restatement of \cref{thm:gen-bound}.
\textit{For any $\delta\in(0,1)$, the population loss of gradually self-trained classifier $h_{T}$ in the target domain is upper bounded with probability at least $1-\delta$ as
\begin{align}
 \eps_{T} (h_{T}) \leq \sum_{t=0}^T q_t \eps_{t}(h_t) + \|\bq_{\scriptscriptstyle n(T+1)}\|_2 \left(1\mathrm{+}\cO\left(\sqrt{\log   (1/ \delta)}\right)\right)
 \mathrm{+}\disc(\bq_{\scriptscriptstyle T+1})\mathrm{+}\cO\left(\sqrt{\log T } \cR_{n(T+1)}^{\mathrm{seq}}(\ell \circ \cH)\right) \nonumber
 \end{align}
 For the class of neural nets considered in Example \ref{example:neural-net}, 
 \begin{align}
 \eps_{T}(h_T) \leq  \eps_{0}(h_0) \mathrm{+} \cO\biggl(T\Delta \mathrm{+}\frac{T}{\sqrt n} \mathrm{+} T\sqrt{\frac{\log 1 / \delta}{n}}\mathrm{+}\frac{1}{\sqrt{nT}}
 \mathrm{+}\sqrt{\frac{(\log nT)^{3L-2}}{nT}} \mathrm{+} \sqrt{\frac{\log 1/\delta}{nT}}~\biggr)\label{eq:supp:gen-bound}
\end{align}}

\paragraph{A Reductive View of the Learning Process of Gradual Self-Training}

If we directly apply Corollary 2 of \citet{kuznetsov2020discrepancy}, we can obtain a generalization bound as
\begin{align}
        \eps_{\mu_T} (h) &\leq \sum_{t=0}^{T} q_t \eps_{\mu_t} (h) + \disc(\bq_{T+1}) + \|\bq_{T+1}\|_2 + 6M \sqrt{4\pi \log T} \mathcal R_T^{\mathrm{seq}} (\ell \circ \mathcal H) \nonumber\\
        &\quad + M \|\bq_{T+1}\|_2 \sqrt{8 \log \frac{1}{\delta }}\nonumber\\
        &\leq \sum_{t=0}^{T} q_t \eps_{\mu_t} (h) + O(T \Delta) + \cO(\frac {1} {\sqrt{T}}) + 6M \sqrt{4\pi \log T} \mathcal R_T^{\mathrm{seq}} (\ell \circ \mathcal H) +  \cO(\sqrt{\frac{\log \frac{1}{\delta }}{T}})\label{eq:supp:proof:gen-bound:naive-bound}
\end{align}
where $M$ is an upper bound on the loss (\cref{lemma:bounded-loss} proves such a $M$ exists), and the last inequality is obtained by setting $\bq_{T+1} = \bq^*_{T+1} = (\frac{1}{T+1},\dots, \frac{1}{T+1})$.

A typical generalization bound involves terms with dependence on $N$ (the training set size), usually in the form $\cO(\sqrt{\frac 1 N})$, and these terms vanish in the infinite-sample limit (i.e., $N\rightarrow \infty$). These terms also appear in standard generalization bounds of unsupervised domain adaptation \citep{ben2007analysis,zhao2019domain}, where $N$ becomes the number of available unlabelled data in the target domain.

In the case of gradual domain adaptation, the total number of available unlabelled is $Tn$, and we would expect $Tn$ will appear in a form similar to $\cO(\sqrt{\frac{1}{nT}})$, which vanishes in the infinite-sample limit (i.e., $nT\rightarrow \infty$). However, the generalization bound \eqref{thm:gen-bound} has terms $\cO(\sqrt{\frac 1 T})$ and $\cO(\sqrt{\frac{\log \frac 1 \delta}{T}})$, which does not vanish even with infinite data per domain, i.e., $n\rightarrow \infty$ (certainly results in $Tn\rightarrow \infty$).

We attribute this issue to the coarse-grained nature of online learning analyses such as \citet{kuznetsov2016time,kuznetsov2020discrepancy}, which do not take data size per domain into consideration.

To address this issue, we propose a novel reductive view of the entire learning process of gradual self-training, leading to a more fined-grained generalization bound than Eq. \eqref{eq:supp:proof:gen-bound:naive-bound}. 

We draw a diagram to illustrate this reductive view in Fig. \ref{fig:our-view}. Specifically, instead of viewing each domain as the smallest element, we zoom in to the sample-level and view each sample as the smallest element of the learning process. We view the gradual self-training algorithm as follows: it has a fixed data buffer of size $n$, and each newly observed sample is pushed to the buffer; the model updates itself by self-training once the buffer is full; after the update, the buffer is emptied. Notice that this view does not alter the learning process of gradual self-training.

With this reductive view, the learning process of gradual self-training consists of $nT$ smallest elements (i.e., each sample is a smallest element), instead of $T$ elements (i.e., each domain is a smallest element) in the view of online learning works \citep{kuznetsov2016time,kuznetsov2020discrepancy}. As a result, terms of order $\cO\left(\sqrt{\frac 1 T}\right)$ in \eqref{eq:supp:proof:gen-bound:naive-bound} becomes $\cO\left(\sqrt{\frac{1}{nT}}\right)$, and terms of order $\cO\left(\frac{T}{n}\right)$ also vanish as $n\rightarrow \infty$. Notably, the upper bounds on the terms $\sum_{t=0}^{T} q_t \eps_{\mu_t} (h)$ and $\disc(\bq_{T+1}) $ in \eqref{eq:naive-gen-bound} do not become larger with this view, since there is no distribution shift within each domain (e.g., the learning process over the first $n$ samples in Fig. \ref{fig:our-view} does not involve any distribution shift, and the iteration $n-1\mapsto n$ incurs a distribution shifts, since the $(n-1)$-th sample is in the first domain while the $n$-th sample is in the second domain).

With this reductive view, we can finally obtain a tighter generalization bound for gradual self-training without the issues mentioned previously.
\begin{proof}
With the inductive view introduced above, we can improve the naive bound \eqref{eq:supp:proof:gen-bound:naive-bound} to
\begin{align*}
 \eps_{\mu_T} (h_{T}) &\leq  \sum_{t=0}^{T} \sum_{i=0}^{n-1} q_{nt+i} \eps_{\mu_t} (h_{T}) + \disc(\bq_{n(T+1)}) + \|\bq_{n(T+1)}\|_2 + 6M \sqrt{4\pi \log nT} \mathcal R_{nT}^{\mathrm{seq}} (\ell \circ \mathcal H)\eq \label{eq:supp:proof:gen-bound:first-line}\\
 &\qquad + M \|\bq_{n(T+1)}\|_2 \sqrt{8 \log \frac{1}{\delta }}\nonumber\\    
 &\leq  \frac {1} {T+1} \sum_{t=0}^{T} \eps_{\mu_t}(h_{T}) +  \rho \sqrt{R^2 + 1}~ \frac{T+1}{2} \Delta  + \frac{1}{\sqrt{nT}}+ 6M \sqrt{4\pi \log nT}R_{nT}^{\mathrm{seq}} (\ell \circ \mathcal H) \nonumber\\
 &\quad + M \sqrt{\frac{8\log 1 / \delta}{ nT}} \\
 &\leq \eps_{\mu_0}(h_0) +  \cO\left(T\Delta+ T\sqrt{\frac{\log 1 / \delta}{n}}+ \frac{1}{\sqrt{nT}} + \rho R\sqrt{\frac{(\log nT)^7}{nT}} + \sqrt{\frac{\log 1/\delta}{nT}}\right )
 \end{align*}
 where $\bq_{n(T+1)}$ is taken as $\bq_{n(T+1)}= \bq_{n(T+1)}^*=(\frac{1}{n(T+1)}), \dots, \frac{1}{n(T+1)})$. We used the following facts when deriving the inequalities above:
 \begin{itemize}
     \item The first term of \eqref{eq:supp:proof:gen-bound:first-line} has the following bound
     \begin{align}\sum_{t=0}^{T} \sum_{i=0}^{n-1} q_{nt+i} \eps_{\mu_t} (h_{T}) &= \frac{1}{T+1} \sum_{t=0}^{T} \eps_{\mu_t}(h_{T})\nonumber\\
     &\leq \eps_{\mu_0}(h_0) +  \cO(T\Delta) + \cO\left(\frac{1}{\sqrt n}+ T\sqrt{\frac{\log 1 / \delta}{n}}\right)
     \end{align}
     which is obtained by recursively apply \cref{lemma:error-diff} and \cref{prop:algorithm-stability} to each term in the summation. For example, the last term in $\sum_{t=0}^{T} \eps_{\mu_t}(h_{T})$ can bounded by \cref{prop:algorithm-stability} as follows
     \begin{align*}
    (\text{By \cref{prop:algorithm-stability}})\quad  \eps_{T}(h_{T}) &\leq \eps_{\mu_{T-1}} (h_{T-1}) + \cO\left(W_p(\mu_T,\mu_{T-1}) + \frac{1}{\sqrt n}+\sqrt{\frac{\log 1 / \delta}{n}}\right)\\
    (\text{Same as the above step})&\leq \ldots \\
    &\leq \eps_{\mu_{0}} (h_{0}) + \cO(T\Delta+ \cO\left(T\sqrt{\frac{\log 1 / \delta}{n}}\right)\eq\label{eq:supp:proof:gen-bound:1st-term:last-term}
    \end{align*}
    and the second last term can be bounded similarly with the additional help of \cref{lemma:error-diff}
    \begin{align*}
         (\text{By \cref{lemma:error-diff}})\quad  \eps_{T-1}(h_{T}) &\leq \eps_{\mu_{T}} (h_{T}) + \cO(W_p(\mu_{T},\mu_{T-1})) \\
     (\text{Apply Eq. \eqref{eq:supp:proof:gen-bound:1st-term:last-term}})&\leq \eps_{\mu_{0}} (h_{0}) + T \Delta + \cO\left( T\sqrt{\frac{\log 1 / \delta}{n}}\right)
    \end{align*}
    All the rest terms (i.e., $\eps_{T-2}(h_{T}),\dots,\eps_{0}(h_{T})$) can be bounded in the same way.
\item The second term of \eqref{eq:supp:proof:gen-bound:first-line} can be bounded by applying \cref{lemma:disc-bound}.
\item The value of $R_{nT}^{\mathrm{seq}} (\ell \circ \mathcal H)$ can be bounded by combining \cref{lemma:seq-rademachor-composite} and Example \ref{example:neural-net}.
 \end{itemize}
\end{proof}

\subsection{Helper Lemmas}\label{supp:proof:helper-lemmas}

\begin{lemma}[Bounded Loss]\label{lemma:bounded-loss}
For any $x\in \mathcal X, y \in \mathcal Y, h\in \mathcal H$, the loss $\ell(x,y)$ is upper bounded by some constant $M$, i.e., $l(h(x), y)\leq M$.
\end{lemma}
\begin{proof}
Notice that i) the input $x$ is bounded in a compact space, specifically, $\|x\|_2 \leq 1$ (ensured by Assumption \ref{assum:input-bound}), ii) $y$ lives in a compact space in $\bR$ (defined in Sec. \ref{sec:setup}), iii) the hypothesis $h\in \mathcal H$ is $R$-Lipschitz, and iv) the loss function $\ell$ is $\rho$-Lipschitz.

Combining these conditions, one can easily find that there exists a constant $M$ such that $l(h(x), y)$ for any $x\in \mathcal X, y\in \mathcal Y, h\in \mathcal H$.
\end{proof}

\begin{lemma}[Lemma 14.8 of \citet{rakhlin2014notes}]\label{lemma:seq-rademachor-composite} For $\rho$-Lipschitz loss function $l$, the sequential Rademacher complexity of the loss class $\ell \circ \mathcal H$ is bounded as
\begin{align}
     \mathcal{R}^{\mathrm{seq}}_T(\ell \circ \mathcal H)  \leq  \cO(\rho \sqrt{(\log T)^3}) \mathcal{R}^{\mathrm{seq}}_T(\mathcal H)
\end{align}
\end{lemma}

\begin{proof}
See \citet{rakhlin2014notes}.
\end{proof}

\subsection{Derivation of the Optimal $T$}\label{supp:proof:optimal-T}

In Sec. \ref{sec:optimal-path}, we show a variant of the generalization bound in \eqref{eq:simplified} as
\begin{equation}
\eps_{T}(h_T) \leq  \eps_{0}(h_0) \mathrm{+} \inf_{\mathcal{P}} \widetilde{\cO}\biggl(T\deltam \mathrm{+}\frac{T}{\sqrt n} \mathrm{+} \sqrt{\frac{1}{nT}}~\biggr)
\end{equation}
where $\deltam$ is an upper bound on the average $W_p$ distance between any pair of consecutive domains along the path, i.e., $\Delta_{\max} \geq \frac{1}{T}\sum_{t=1}^T W_p(\mu_{t-1}, \mu_t)$.

Given that $T,\deltam,n$ are all positive, we know there exists an optimal $T=T^*$ that minimizes the function
\begin{align}
    f(T)\coloneqq T\deltam \mathrm{+}\frac{T}{\sqrt n} \mathrm{+} \sqrt{\frac{1}{nT}}~ ,
\end{align}
and one can straightforwardly derive that
\begin{align}
    T^* = \left(\frac{1}{2(1+\deltam \sqrt{n}~)}\right)^{\frac 2 3}~.
\end{align}
\begin{proof}
The derivative of $f(T)$ is
\begin{align}
f'(T) = \deltam + \frac{1}{\sqrt n} - \frac{1}{2\sqrt{n}}T^{-\frac 3 2}~,
\end{align}
and the second-order derivative of $f(T)$ is
\begin{align}\label{eq:supp:optimal-T:2nd-grad}
f''(T) =   \frac{3}{4\sqrt{n}}T^{-\frac 5 2}~.
\end{align}
Eq. \eqref{eq:supp:optimal-T:2nd-grad} indicates that $f(T)$ is strictly convex in $T\in(0,\infty)$. Then, we only need to solve for the equation
\begin{align}
    f'(T) = 0 
\end{align}
as $T\in (0,\infty)$, which gives our the solution
\begin{align}
    T^* = \left(\frac{1}{2(1+\deltam \sqrt n~)}\right)^{\frac 2 3} ~.
\end{align}
\end{proof}

\section{Experimental Details}

\textbf{Implementation.} We adopt the code of \citep{kumar2020understanding} from \url{https://github.com/p-lambda/gradual_domain_adaptation}. We make necessary modifications to include two new datasets (Color-Shift MNSIT and CoverType) to the codebase. Also, we directly use the model structure and hyper-parameters provided in this codebase. 
\\

\textbf{Code.} Our code is provided in

\url{https://github.com/Haoxiang-Wang/gradual-domain-adaptation}.


\end{document}